\documentclass[runningheads]{llncs}
\usepackage[T1]{fontenc}
\usepackage{graphicx}
\usepackage{amsmath}
\usepackage{algorithm}
\usepackage{bm}
\usepackage{pdflscape}
\usepackage{comment}
\usepackage{mathabx}
\newcommand{\vu}{\bm{u}}

\newcommand{\vx}{\bm{x}}

\newcommand{\ds}{\displaystyle}

\begin{document}

\title{Assessing Group Fairness with\\
	Social Welfare Optimization}
\author{Violet (Xinying) Chen\inst{1}
\and J.\ N. Hooker\inst{2} 
\and Derek Leben\inst{2}
}
\authorrunning{V. Chen et al.}
\institute{Stevens Institute of Technology
\email{vchen3@stevens.edu}
\and
Carnegie Mellon University
	\email{jh38@andrew.cmu.edu,dleben@andrew.cmu.edu}
}
\maketitle              % typeset the header of the contribution
\begin{abstract}
Statistical parity metrics have been widely studied and endorsed in the AI community as a means of achieving fairness, but they suffer from at least two weaknesses.  They disregard the actual welfare consequences of decisions and may therefore fail to achieve the kind of fairness that is desired for disadvantaged groups.  In addition, they are often incompatible with each other, and there is no convincing justification for selecting one rather than another.  This paper explores whether a broader conception of social justice, based on optimizing a social welfare function (SWF), can be useful for assessing various definitions of parity.  We focus on the well-known alpha fairness SWF, which has been defended by axiomatic and bargaining arguments over a period of 70 years.  We analyze the optimal solution and show that it can justify demographic parity or equalized odds under certain conditions, but frequently requires a departure from these types of parity.  In addition, we find that predictive rate parity is of limited usefulness.  These results suggest that optimization theory can shed light on the intensely discussed question of how to achieve group fairness in AI.

\keywords{Social welfare optimization  \and group parity in AI }
\end{abstract}

\section{Introduction}

There is growing demand within industry and government for assurance that machine learning (ML) models respect and promote equality of impact across protected groups \cite{Fjeld,Jobin} and comply with legal requirements \cite{Feldman,Selbst}.  This concern arises in contexts that range from hiring and parole decisions to mortgage lending and credit ratings.  One prominent method of satisfying these ethical and legal goals is the use of statistical parity metrics \cite{barocas}. 
For example, one might assess two groups have equal approval rates (\emph{demographic parity}), whether the approval and rejection rates of qualified candidates are equal ({\em equalized odds}), or whether the fraction of qualified candidates among those approved is the same (\emph{predictive rate parity}).

There are at least two problems, however, with reliance on statistical parity as a measure of fairness.  One is that parity metrics take no account of the actual utility consequences of being selected or rejected.  Presumably, group disparities are viewed as unjust because different groups derive unequal benefits from the selection process.  Yet an assessment of these benefits requires consideration of the actual welfare outcomes of selecting or rejecting individuals.  For example, rejecting a member of a disadvantaged group may have greater negative consequences than rejecting a member of an advantaged group.  The standard parity metrics take account only of the number of individuals selected or rejected, not the impacts of these decisions.

A second problem is that parity metrics are frequently incompatible with each other \cite{chouldechova2017fair,Friedler,Kleinberg} and, in particular, imply different trade-offs between fairness and accuracy \cite{Bertsimas,Kleinberg}.  As a result, there is often no consensus on which metric is appropriate in a given context.  This is illustrated by the famous debate over parole decisions between ProPublica and Northpointe (now Equivant) regarding whether the latter's COMPAS product is fair, with one side claiming that the model is unfair because it fails to achieve equalize odds, and the other side claiming it is fair because it achieves predictive rate parity \cite{Propublica2016,DieMenBre16}.  Lacking any further grounds for settling this dispute, the debate has (for now) reached a stalemate.  Ideally, one would justify (or reject) a parity metric by appealing to a broader principle of justice.  

\begin{comment}
In addition to the technical trade-offs, there are three important philosophical problems with the use of statistical parity metrics:

\begin{enumerate}
    \item \emph{Accounting for welfare.} Group disparity is presumably undesirable because it results in an unequal distribution of benefits.  . 
    
    \item \emph{Protecting relevant groups.} Parity metrics do not specify which groups should be protected for fairness purposes.  When multiple disadvantaged groups have conflicting interests, there is no guidance for how to reconcile them.  Ideally, one would promote parity across all groups by achieving some kind of justice in the population as a whole. 
    
    \item \emph{Selecting and justifying parity metrics.} Given the technical incompatibilities described above, there is no consensus on which parity metrics should be used, and how they can be justified.  Ideally, one would justify a parity measure by appealing to a broader principle of justice.

\end{enumerate}

\end{comment}

In this paper, we explore an approach for evaluating group parity metrics via their effects on the welfare of individuals in each group. The aim is to connect the debate about group parity with the rich tradition of welfare economics, where policies are evaluated by their effects on social welfare, as measured by a social welfare function (SWF).  Such a function can take into account the distribution of utilities as well as overall welfare.  We ask whether a selection policy that optimizes social welfare, as measured by a SWF, results in some particular form of group parity or requires departure from the standard parity measures.  Our underlying hypothesis is that insights obtained from optimization theory can shed light on the vexing problem of fairness in AI.  

%To this end, we construct {\em single-policy} and {\em dual-policy} models to assess  implications of general social welfare for group parity.  Both models address problem~1 by considering utilitarian outcomes.  The single-policy model addresses problem~2 by being ``blind'' to group membership.  It explores the extent to which achieving fairness in the general population, as measured by a SWF, results in parity across multiple groups.  The dual-policy model addresses problem~3 by allowing different selection policies for a protected group and the rest of the population.  It attempts to determine whether SWF-based decision making can provide justification for particular parity metrics.  The parity metrics we consider are demographic parity, equalized odds (with a focus on the positive acceptance rate, or equality of opportunity), and predictive rate parity.
	
As a first step in this research program, we investigate the parity implications of {\em alpha fairness} \cite{MoWal00,VerAyeBor10}, a well-known family of SWFs parameterized by a nonnegative real number $\alpha$.  Larger values of $\alpha$ indicate a stronger emphasis on fairness as opposed to maximizing total utility, the latter corresponding to $\alpha=0$.  Alpha fairness can therefore evaluate the trade-off of fairness and accuracy, a perennial issue in machine learning.  Other special cases include the maximin (Rawlsian) criterion ($\alpha=\infty$) and {\em proportional fairness}, also known as the Nash bargaining solution ($\alpha=1$). We ask what are the parity implications of a given level of fairness as indicated by $\alpha$.

Our purpose here is not to defend alpha fairness as a fairness criterion, but to explore the implications of a criterion that has {\em already} been extensively defended.  Alpha fairness in its various forms has been studied for over 70 years by investigators that include two Nobel laureates (John Nash and J.\ C.\ Harsanyi).  Nash \cite{nash1950bargaining} gave an axiomatic argument for his bargaining solution in 1950, while Rubinstein, Harsanyi and Binmore \cite{rubinstein1982perfect,harsanyi1986rational,binmore1986nash} supplied bargaining arguments.  Lan et.\ al \cite{LanChi11,LanKaoChiSab10} provided an axiomatic derivation for general alpha fairness and proposed an interpretation of the $\alpha$ parameter.  Bertsimas et al.\ \cite{Bertsimas} studied resulting equity/efficiency trade-offs.  Alpha fairness has also seen a number of practical applications, particularly in telecommunications and other engineering fields \cite{Kelly1998,Mazumdar1991,MoWal00,OgrLusPioNacTom14,VerAyeBor10}.

%We should also clarify that we do not advocate alpha fairness as a criterion for deciding who is accepted or rejected.  Just as parity metrics are used to {\em evaluate} the fairness of decisions that result from AI technology, we similarly view alpha fairness as an assessment tool that is applied after the fact.  

After a brief survey of related work, we first establish a general solution to the problem of maximizing alpha fairness subject to a constraint on the number of individuals selected.  We then present a utility model that allows us to relate group characteristics to the implications of alpha fairness.  Following this, we describe specific implications for demographic parity, equalized odds, and predictive rate parity, and draw conclusions from these results.  
%To our knowledge, this erpresents the first effort to connect social welfare functions with group parity measures.
  
%with the ultimate aim of providing intellectual support for justifable parity measures.
%with the ultimate aim of replacing various intuitions about which parity metrics are justified with a foundation of reasoned argument.    

\section{Related Work} \label{sec:related}

Statistical group parity metrics are the most widely studied approach to fairness in AI and machine learning.  Much of this research is surveyed in \cite{CasCruGreRegPenCos22,mehrabi2021survey}.  However, a welfarist approach is beginning to receive recognition in AI fairness literature, e.g. \cite{Hu2019,hu2020fair,CorbettDavies2018TheMA,loi2019philosophical,Card,ChenHooker22,CheHoo20,Baumann}.  One motivation is pragmatic: social welfare can provide a ``common currency'' with which one can justify the choice of parity metric, when the typical justifications are incommensurate \cite{greene}. For example, arguments for individual fairness appeal to procedural justice concerns, while arguments for group fairness appeal to distributive justice \cite{Binns,Leben}.  When a model cannot satisfy both of these values, it is necessary to justify one's choice. 
%Appealing to welfare can help to cash out these claims in a common parlance. 
Another motivation for a welfarist approach is ethical:
%rather than considering group parity as having \emph{inherent} value, it may instead have \emph{instrumental} value. 
one may wish to strive for group parity to make disadvantaged groups better off, rather than to achieve equality for its own sake \cite{Carter,Moss}.

Social welfare functions have been used in optimization models for some time, as surveyed in \cite{Ogryczak2008,Karsu2015,ChenHooker22}.  Aside from alpha fairness, SWFs that balance equity and efficiency include Kalai-Smorodinsky bargaining \cite{KalSmo75} and threshold functions \cite{WilliamsCookson2000,HooWil12,CheHoo21}.  
%Their analysis is a logical next step after alpha fairness in the research program initiated in this paper.  
%Among the large number of SWFs in literature, alpha fairness is the best known class that balances fairness and efficiency as a general welfare measure. Alpha fairness therefore can evaluate the fairness and accuracy (which is correlated with efficiency) trade-off, a perennial issue in machine learning. The fairness-efficiency trade-off is one of the oldest problems in social choice theory and has been studied by two Nobel laureates (John Nash and JC Harsanyi), among others.  A special case of alpha fairness ($\alpha = 1$), known as proportional fairness or the Nash bargaining solution, is the oldest solution to this problem. \citet{nash1950bargaining} gave it an axiomatic justification in 1950, and \citet{harsanyi1986rational,rubinstein1982perfect,binmore1986nash} gave bargaining arguments. More recently, \citet{lan2010axiomatic} provided axiomatic derivation for alpha fairness and interpretation for the $\alpha$ parameter. 

Despite the large literature on SWFs and group parity metrics, we describe here what is, to our knowledge, the first explicit connection between them. 
%The results of our paper can enable parity metrics to be analyzed in terms of preexisting work on SWFs, and allow some important inferences to be drawn about the conditions under which fairness metrics may make members of a protected group worse-off. 
%The most relevant literature include: 
%For example, the results of Predictive Parity in both single-policy and dual-policy model types will create worse conditions for members of both advantaged and disadvantaged groups, compared with maximizing for accuracy. 

\section{The Basic Model} \label{sec:prob-setting}

We address the task of selecting individuals from a population to receive a benefit or resource, such as a mortgage loan or a job interview.  Some individuals belong to a protected group that is disadvantaged with respect to qualification status. We define binary variables $D, Y, Z$ to indicate whether an individual is selected ($D=1$), qualified ($Y=1$), or protected ($Z=1$). 
%We will refer to an individual predicted to be qualified as {\em apparently qualified}.  
To simplify notation, we use $D$ to represent $D=1$ and $\neg D$ to represent $D=0$, and similarly for $Y$ and $Z$.  

We have demographic parity when $P(D|Z)=P(D|\neg Z)$, equalized odds (in the positive sense of equality of opportunity) when $P(D|Y,Z)=P(D|Y,\neg Z)$, and predictive rate parity when $P(Y|D,Z)=P(Y|D,\neg Z)$.  We interpret the conditional probability $P(D|Z)$ as the fraction of protected individuals who are selected, and similarly for the other probabilities.  The latter two types of parity are typically defined in terms of qualifications that are determined after the fact, such as whether a mortgage recipient repaid the loan, a job interviewee was hired, or a parolee committed no further crimes.  In addition, calculation of the odds ratio requires knowledge of how many rejected candidates are qualified.  
%and not whether the applicant appeared to be qualified at the time of decision.  

To assess utilitarian outcomes, we suppose that an individual $i$ experiences expected utility $u_i=a_i+b_i$ if selected, and a baseline utility $u_i=b_i$ if rejected.  We refer to $a_i$ as the {\em selection benefit}.  It can be negative (indicating that selection is harmful), but we assume that $b_i>0$ and $a_i+b_i>0$ because alpha fairness is not defined for nonpositive utilities.  This assumption can be met by a positive translation of the utility scale if necesssary.  
\begin{comment}

\begin{table}[!h]
\centering
\caption{Utility parameters for individuals} \label{ta:utility}
\begin{tabular}{rcccccc}
& \parbox{10ex}{\begin{center} \vspace{1ex} Qualified population \end{center}}  
& \parbox{10ex}{\begin{center} Unqualified population \end{center}}                                          
& \parbox{10ex}{\begin{center} Qualified majority \end{center}}               
& \parbox{12ex}{\begin{center} Unqualified majority \end{center}} 
& \parbox{10ex}{\begin{center} Qualified minority \end{center}}               
& \parbox{12ex}{\begin{center} Unqualified minority \end{center}} \\
\hline
\ \\[-1.5ex]
Selected     & $a_1+b_1$ & $a_0+b_0$ & $a^M_1+b^M_1$ & $a^M_0+b^M_0$ & $a^m_1+b^m_1$ & $a^m_0+b^m_0$ \\[0.5ex]
Rejected     & $b_1$     & $b_0$     & $b^M_1$      & $b^M_0$      & $b^m_1$       & $b^m_0$ \\[0.5ex]
\hline
\end{tabular}
\end{table}

\end{comment}

We assess the desirability of a utility distribution $\vu=(u_1,\ldots,u_n)$ with the alpha fairness social welfare function, given by 
\begin{equation}
W_{\alpha}(\vu) = \left\{
\begin{array}{ll}
{\ds
	\frac{1}{1-\alpha}\sum_i u_i^{1-\alpha}, 
} & \mbox{if}\; \alpha\geq 0 \;\mbox{and}\; \alpha\neq 1 \\
{\ds 
	\sum_i \log(u_i), 
} & \mbox{if} \; \alpha=1
\end{array}
\right.
\label{eq:alpha}
\end{equation}
Alpha fairness is achieved by maximizing $W_{\alpha}(\vu)$ subject to a limit on the number of individuals that can be selected.  

We let binary variable $x_i=1$ when individual $i$ is selected.  The expected utility gained by individual $i$ is therefore $a_ix_i+b_i$.  The social welfare resulting from a given vector $\vx=(x_1,\ldots,x_n)$ of selection decisions, as measured by the alpha fairness SWF, is 
\begin{equation}
W_{\alpha}(\vx) = 
\left\{
\begin{array}{ll}
{\ds
\frac{1}{1-\alpha} \sum_{i=1}^n (a_ix_i + b_i)^{1-\alpha},
} & \mbox{if}\; \alpha\geq 0\;\mbox{and}\;\alpha\neq 1
\\
{\ds
\sum_{i=1}^n \log(a_ix + b_i),
} & \mbox{if}\;\alpha=1
\end{array}
\right.
\label{eq:swf1}
\end{equation}
If $m\;(< n)$ individuals are to be selected, one achieves alpha fairness for a given $\alpha$ by maximizing $W_{\alpha}(\vx)$ subject to 
$\sum_{i=1}^n x_i = m$.  
A maximizing vector $\vx$ can be deduced using a simple greedy algorithm.  We first consider the case $\alpha\neq 1$.  The top expression in \eqref{eq:swf1} can be written as
\begin{equation}
\frac{1}{1-\alpha} \sum_{i=1}^n b_i^{1-\alpha} +
\frac{1}{1-\alpha} \sum_{i=1}^n \Big( (a_ix_i+b_i)^{1-\alpha} - b_i^{1-\alpha} \Big)
\label{eq:swf4}
\end{equation}
Since the first term is a constant, we can maximize \eqref{eq:swf4} by maximizing its second term, which can be written as
\begin{equation}
\frac{1}{1-\alpha} \sum_{i|x_i=1} \hspace{-1ex} \Big( (a_i+b_i)^{1-\alpha} - b_i^{1-\alpha} \Big)
=
\sum_{i|x_i=1} \hspace{-1.1ex} \Delta_i(\alpha)
\label{eq:swf2}
\end{equation}
where we define
\[
\Delta_i(\alpha) = 
\left\{
\begin{array}{ll}
\frac{1}{1-\alpha}{\ds \big((a_i+b_i)^{1-\alpha}
- b_i^{1-\alpha}\big)}, & \mbox{if} \; \alpha\geq 0, \; \alpha\neq 1 \\[1ex]
\log(a_i+b_i)-\log(b_i), & \mbox{if}\;\alpha=1
\end{array}
\right.
\]
The term $\Delta_i(\alpha)$ is the increase in welfare that results from selecting individual $i$ (for a given $\alpha\neq 1$).  We can maximize \eqref{eq:swf2} subject to $\sum_{i=1}^n x_i = m$ by selecting the $m$ individuals with the largest {\em welfare differential} $\Delta_i(\alpha)$.  A similar argument applies for $\alpha=1$.  Thus we have

\begin{theorem} \label{th:1}
	If $\Delta_{\pi_1}(\alpha)\geq \cdots\geq\Delta_{\pi_n}(\alpha)$, where $\pi_1,\ldots,\pi_n$ is a permutation of $1,\ldots,n$, then one can maximize $W_{\alpha}(\vx)$ subject to $\sum_{i=1}^n x_i = m$ by setting $x_i=1$ for $i=\pi_1,\ldots,\pi_m$, and $x_i=0$ for $i=\pi_{m+1},\ldots,\pi_n$.
\end{theorem}
At this point we can easily check whether achieving alpha fairness results in the various forms of group parity by observing whether their definitions are satisfied when individuals $\pi_1,\ldots,\pi_m$ are selected.

\section{Modeling Protected and Nonprotected Groups}

While Theorem~\ref{th:1} specifies an alpha fair selection policy for any given set of individual utility parameters $(a_1,b_1),\ldots(a_n,b_n)$, it yields limited insight into how the utility characteristics of protected and nonprotected groups affect alpha fair selections.   In addition, the large number of parameters makes relationships difficult to analyze in a comprehensible fashion.  
%This will allow us to specify rather succinctly the parameter values for which an alpha fair solution achieves group parity.  

We address these issues by supposing that the expected utilities in the two groups occur on sliding scale.  Specifically, we suppose that the selection benefits $a_i$ in the nonprotected group are distributed uniformly on a scale from a maximum $A_{\max}$ down to a minimum $A_{\min} (< A_{\max})$, and selection benefits in the protected group vary uniformly from $a_{\max}$ down to $a_{\min} (< a_{\max})$.  A nonuniform distribution is more realistic, but it requires a complicated analysis that is harder to interpret, while yielding basically the same qualitative results.  To further simplify analysis, we suppose that the base utility has the same value $B$ for all nonprotected individuals, and the same value $b$ for all protected individuals.  We assume that $B>b$ and, consistent with the previous section,   that $A_{\min}+B>0$ and $a_{\min}+b>0$.  Finally, we suppose that the protected group comprises a fraction $\beta$ of the population, with $0<\beta<1$.

We further assume that individuals within a given group are selected in decreasing order of their selection benefit.  Thus if a fraction $S$ of nonprotected individuals are selected, the last individual selected in that group has the selection benefit $A(S)=(1-S)A_{\max}+ S A_{\min}$ and a social welfare differential of 
\[
\Delta_S(\alpha) = 
\left\{
\begin{array}{ll}
\frac{1}{1-\alpha} \Big( \big(A(S) + B\big)^{1-\alpha} - B^{1-\alpha} \Big), & \mbox{if}\;\alpha\geq 0, \;\alpha\neq 1  \\[1.5ex]
\log\big(A(S)+B\big) - \log(B), & \mbox{if} \; \alpha=1 
\end{array}\right.
\]
Similarly, if a fraction $s$ of individuals are selected in the protected group, the last individual selected has the selection benefit $a(s) = (1-s)a_{\max} + s a_{\min}$
and the social welfare differential
\[
\Delta'_s(\alpha) = 
\left\{
\begin{array}{ll}
\frac{1}{1-\alpha} \Big( \big(a(s) + b\big)^{1-\alpha} - b^{1-\alpha} \Big), & \mbox{if}\;\alpha\geq 0, \;\alpha\neq 1  \\[1.5ex]
\log\big(a(s)+b\big) - \log(b), & \mbox{if} \; \alpha=1 
\end{array}\right.
\]
We will suppose that the population is large enough that $S$ and $s$ can be treated as continuous variables.  This simplifies the analysis considerably without materially affecting the conclusions.

Since the social welfare differential is a monotone increasing function of the selection benefit, selecting individuals in order of decreasing welfare differential is, within each group, the same as selecting in order of decreasing selection benefit.  By Theorem~\ref{th:1}, selection in order of decreasing welfare differential maximizes the alpha fairness SWF subject to $\sum_i x_i=m$ if we select individuals until the desired fraction $\sigma=m/n$ of the population is selected.  This occurs when 
\begin{equation}
(1-\beta)S + \beta s = \sigma, \;\;\mbox{or} \;\;
s = s(S) = \frac{\sigma - (1-\beta)S}{\beta}
\label{eq:alpha29}
\end{equation}

We first take note of the ranges within which $S$ and $s$ can vary, subject to \eqref{eq:alpha29}.  Since we must have $0\leq S\leq 1$ and $0\leq s\leq 1$, $S$ can vary in the range from $S_{\min}$ to $S_{\max}$, where
\[
S_{\min} = \max\Big\{ 0,\; \frac{\sigma-\beta}{1-\beta}\Big\}, \;\; 
S_{\max} = \min\Big\{ 1, \; \frac{\sigma}{1-\beta}\Big\}
\]
and $s$ can vary from $s(S_{\max})$ to $s(S_{\min})$.
Now since $A_{\max}>A_{\min}$, $\Delta_S(\alpha)$ is monotone decreasing in $S$.  Similarly, $\Delta'_s(\alpha)$ is monotone decreasing in $s$, so that $\Delta'_{s(S)}(\alpha)$ is monotone increasing in $S$.  This means that we can consider three cases, illustrated by Fig.~\ref{fig:Cases}:
\begin{description}
	\item (a) $\Delta_{S_{\min}}(\alpha)> \Delta'_{s(S_{\min})}(\alpha)$ and $\Delta_{S_{\max}}(\alpha)\geq \Delta'_{s(S_{\max})}(\alpha)$. \smallskip
	\item (b) $\Delta_{S_{\min}}(\alpha)\leq \Delta'_{s(S_{\min})}(\alpha)$ and $\Delta_{S_{\max}}(\alpha)< \Delta'_{s(S_{\max})}(\alpha)$.  \smallskip
	\item (c) $\Delta_{S_{\min}}(\alpha)>\Delta'_{s(S_{\min})}(\alpha)$ and
	$\Delta_{S_{\max}}(\alpha)<\Delta'_{s(S_{\max})}(\alpha)$
\end{description}

\begin{figure}[!h]
	\centering
	\includegraphics[scale=0.16,clip=true,trim=80 10 30 30]{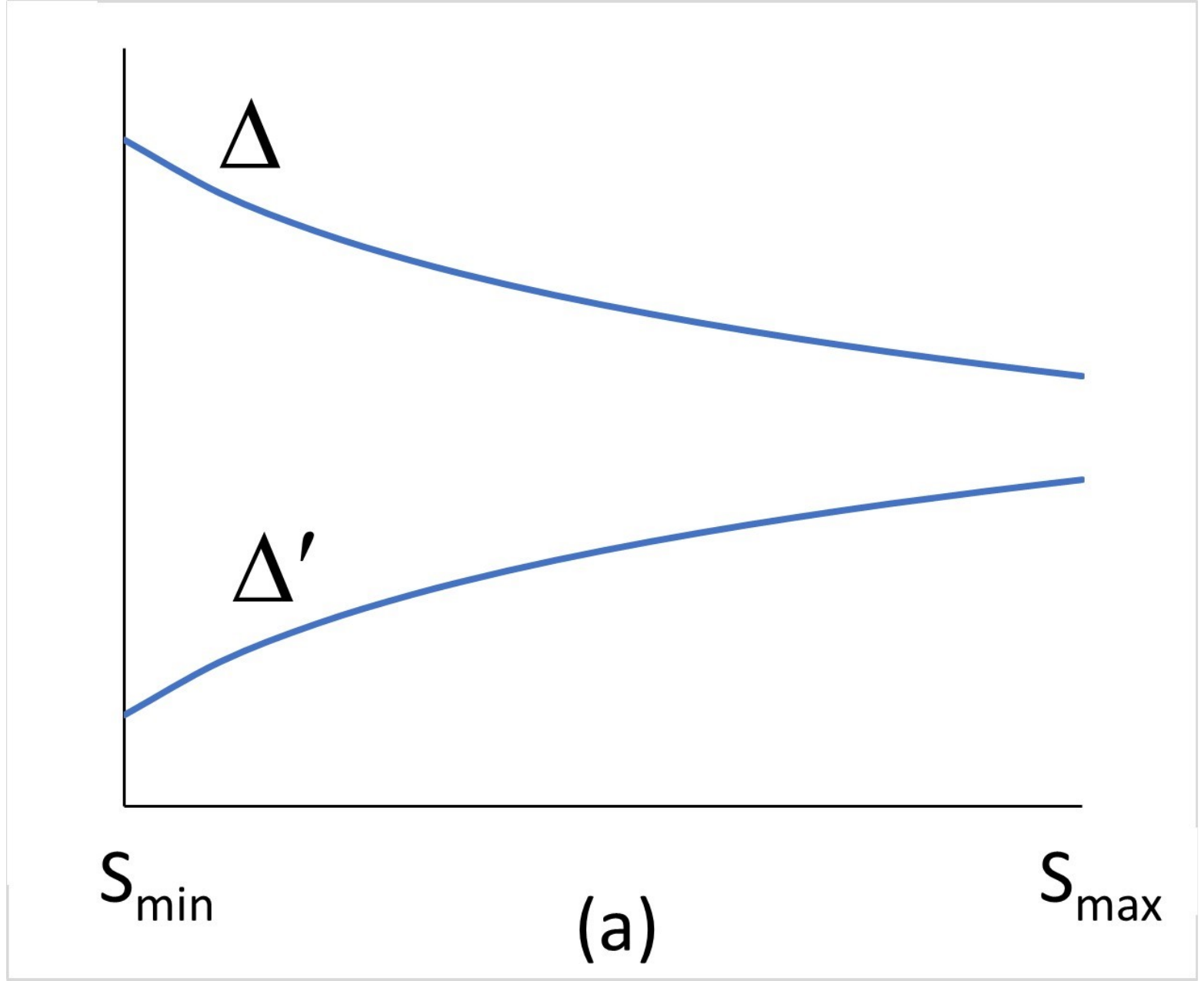} 
	\includegraphics[scale=0.16,clip=true,trim=80 10 30 30]{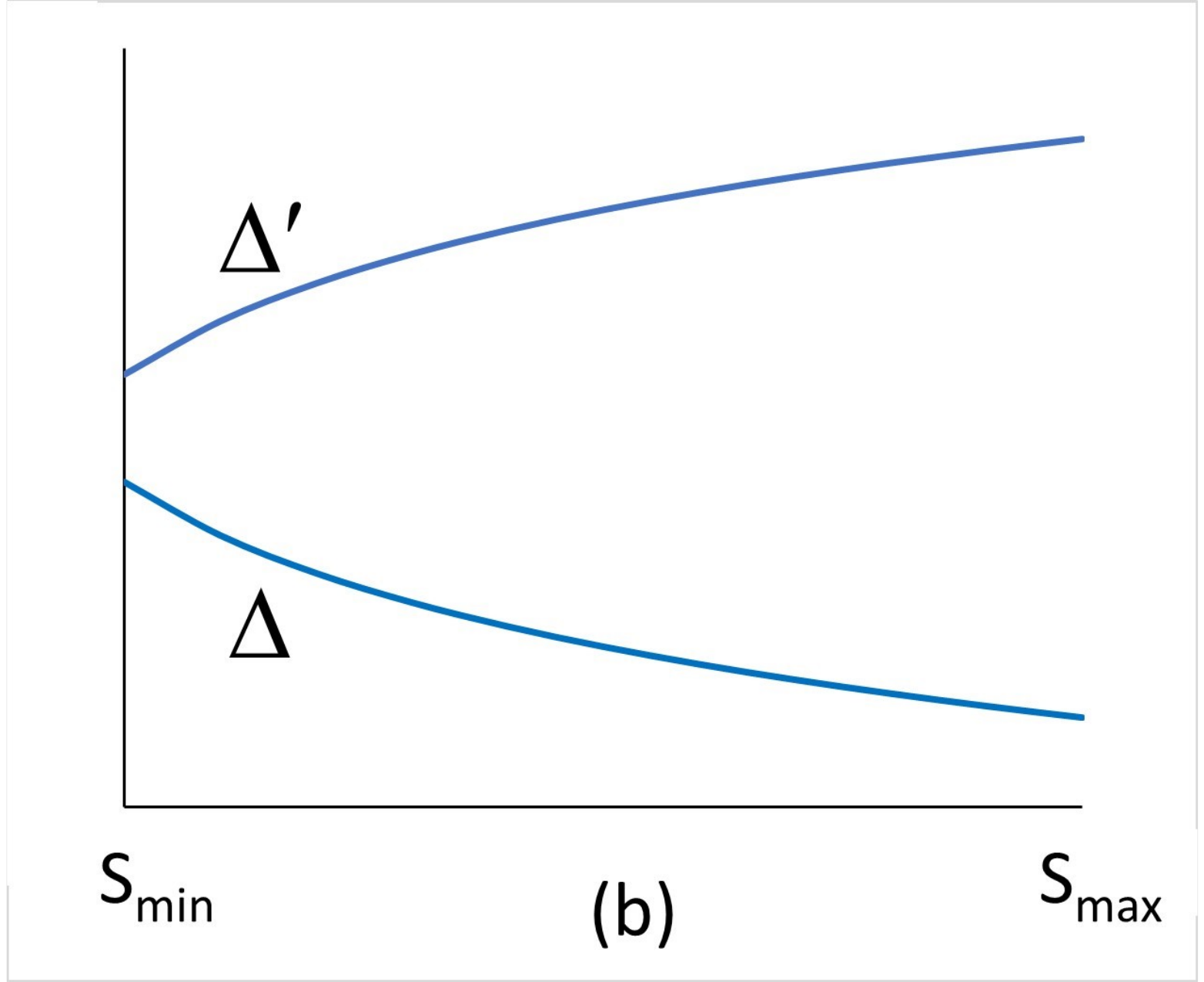} 
	\includegraphics[scale=0.16,clip=true,trim=80 10 30 30]{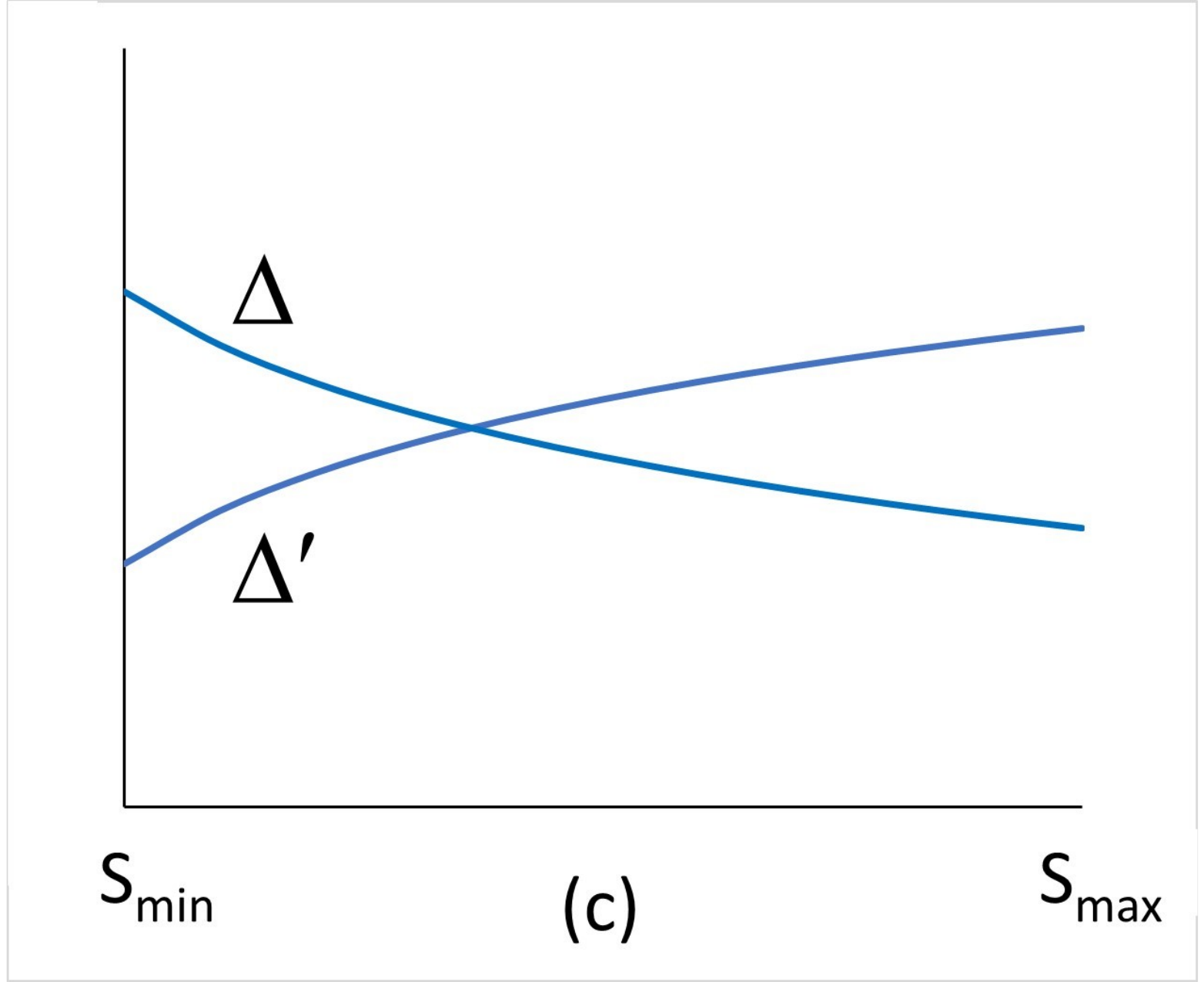} 
	\vspace{-2ex}
	\caption{Cases (a), (b), and (c) in the proof of Theorem \ref{th:alpha}} \label{fig:Cases}
\end{figure}

\begin{theorem} \label{th:alpha}
	Suppose that individuals are selected in decreasing order of their selection benefit, and let $S^*$ and $s^*=s(S^*)$, respectively, be the fraction of the nonprotected and protected groups selected at the end of the selection process.  Then for a sufficiently large population, $S^*$ and $s^*$ achieve alpha fairness if and only if 
	\begin{equation}
	\left\{
	\begin{array}{ll}
	(S^*,s^*) = 
	{\ds
		\Big( \min\Big\{1,\frac{\sigma}{1-\beta}\Big\}, \;\; 
		\frac{\sigma}{\beta}\Big[1 - \min\Big\{1,\frac{1-\beta}{\sigma}\Big\} \Big] \Big),
	} & \mbox{in case (a)} \\[2ex]
    (S^*,s^*) = 
	{\ds
		\Big( \frac{\sigma}{1-\beta}\Big[1 - \min\Big\{1, \frac{\beta}{\sigma}\Big\} \Big], \;\;
		\min\Big\{1, \frac{\sigma}{\beta}\Big\} \Big), 
	} & \mbox{in case (b)} \\[2.5ex]
	\Delta_{S^*}(\alpha)=\Delta'_{s(S^*)}(\alpha) , & \mbox{in case (c)}
	\end{array} 
	\right.
	\label{eq:alpha35}
	\end{equation} 
\end{theorem}

\begin{proof}
	Recall that by Theorem~1, alpha fairness is achieved by selecting individuals in decreasing order of their welfare differential until $S=s(S)$.  We consider the three cases separately.  (a) Because $\Delta_S(\alpha)\geq\Delta'_{s(S)}(\alpha)$ for all $S\in[S_{\min},S_{\max}]$, we select entirely from the nonprotected group until it is exhausted, and then move to the protected group if necessary to select a fraction $\sigma$ of the population. Thus we can set
	\[
	S^*=\min\Big\{S_{\max},\frac{\sigma}{1-\beta}\Big\} = \min\Big\{1,\frac{\sigma}{1-\beta}\Big\}
	\]
	where the first equality is due to the fact that we must have $S^*=\sigma/(1-\beta)$ in order to select a fraction $\sigma$ if the population if $\sigma\leq 1-\beta$, and the second equality is due to the definition of $S_{\max}$.  The expression given in \eqref{eq:alpha35} for $s^*=s(S^*)$ follows directly from the definition of $s(S^*)$, and it is easily checked that $s_{\min}\leq s^*\leq s_{\max}$ using the definitions of $s_{\min}$ and $s_{\max}$.  (b) The argument is very similar to that of the previous case.  (c) In this case, some but not all individuals are selected in both groups.  Let $(S,s)$ be the fraction of the nonprotected and protected individuals selected at any given point in the selection process.  We first show that $\Delta_S(\alpha)=\Delta'_{s}(\alpha)$ for a sufficiently large population.  Let $\Delta_0$ and $\Delta_1$ be the welfare differentials of the last two nonprotected individuals selected, and $\Delta'_0$ and $\Delta'_1$ the differentials of the last two protected individuals selected.  Their selection order is necessarily one of the following:  $(\Delta_0,\Delta'_0,\Delta_1,\Delta'_1)$, $(\Delta_0,\Delta'_0,\Delta'_1,\Delta_1)$, $(\Delta'_0,\Delta_0,\Delta_1,\Delta'_1)$, $(\Delta'_0,\Delta_0,\Delta'_1,\Delta_1)$. In each case, $|\Delta_1-\Delta'_1|$ is at most \mbox{$\max\{\Delta_0-\Delta_1,\Delta'_0-\Delta'_1\}$}.   For a sufficiently large population, $\Delta_0-\Delta_1$ and $\Delta'_0-\Delta'_1$ are arbitrarily small, and so $|\Delta_1-\Delta'_1|$ is arbitrarily small.  Thus we have $\Delta_S(\alpha)=\Delta'_s(\alpha)$ throughout the selection process, and in particular at the end of the process, when $(S,s)=(S^*,s^*)$.  The theorem follows.  $\Box$
\end{proof}

\begin{figure}[!h]
	\centering
	\includegraphics[scale=0.39,clip=true,trim=20 180 10 60]{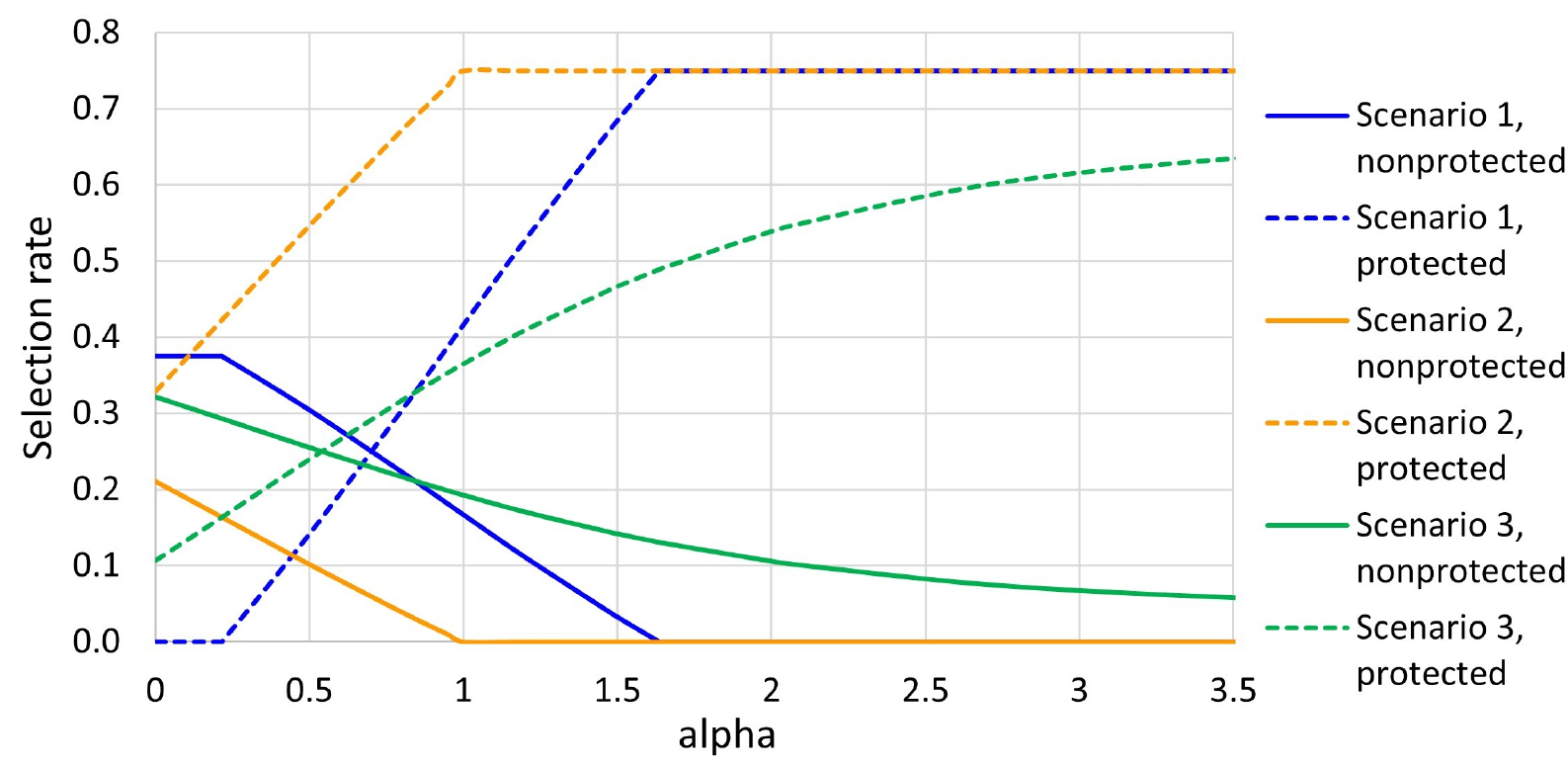} 
	\vspace{-3ex}
	\caption{Alpha fair selection rates, assuming overall selection rate of 0.25} \label{fig:Demo025}
\end{figure}

\begin{figure}[!h]
	\centering
	\includegraphics[scale=0.39,clip=true,trim=20 140 10 120]{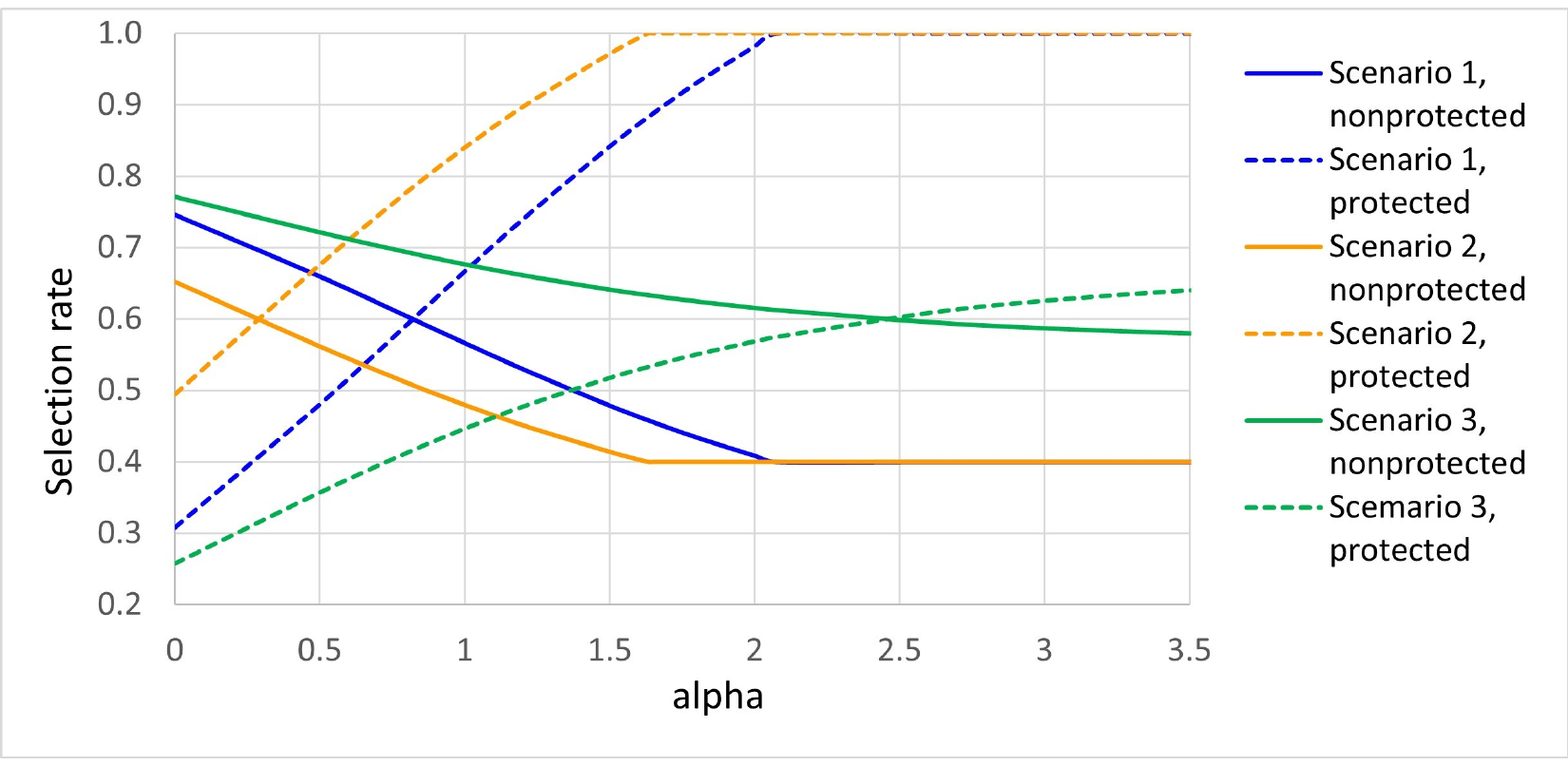} 
	\vspace{-2.5ex}
	\caption{Alpha fair selection rates, assuming overall selection rate of 0.6} \label{fig:Demo06}
\end{figure}

\begin{figure}[!h]
	\centering
	\includegraphics[scale=0.39,clip=true,trim=25 260 5 25]{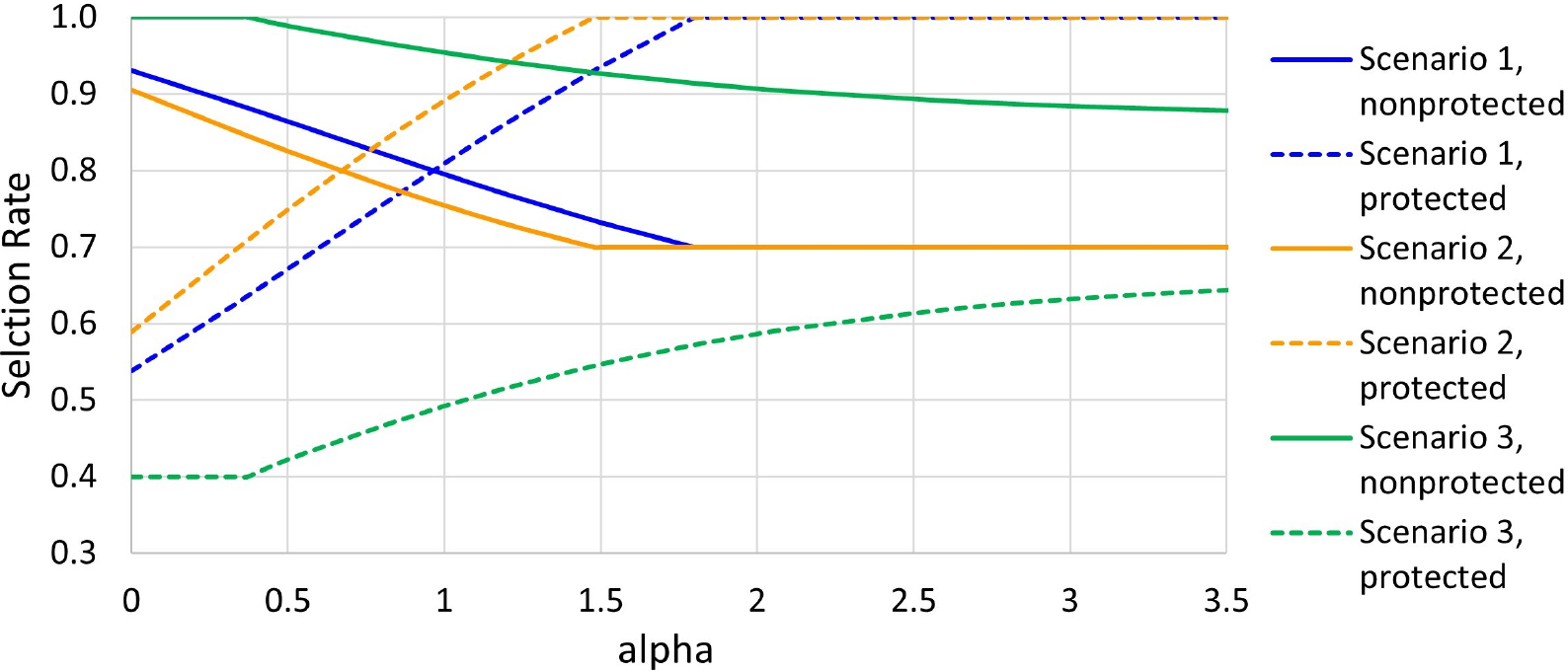} 
	\vspace{-2.5ex}
	\caption{Alpha fair selection rates, assuming overall selection rate of 0.8} \label{fig:Demo08}
\end{figure}

To explore how alpha fair selection policies depend on the utility characteristics of protected and nonprotected groups, we define three scenarios that represent qualitatively different practical situations.
\begin{description}
	\item {\em Scenario 1.}  Protected individuals are somewhat less likely to benefit from being selected, as when those selected for job interviews are less likely to be hired due to less obvious qualifications.  Here, $[A_{\min},A_{\max}]=[0.5,1.5]$ and $[a_{\min},a_{\max}]=[0.2,1]$.  
	\item {\em Scenario 2.} Some protected individuals can benefit more than anyone else from selection, as when talented but economically disadvantaged individuals are admitted to a university.  Here, $[A_{\min},A_{\max}]=[0.5,0.8]$ and $[a_{\min},a_{\max}]=[0.2,1]$.
	\item {\em Scenario 3.}  Significantly many protected individuals are likely to be harmed by selection, as when failure to repay a mortgage results in eviction.  Here, $[A_{\min},A_{\max}]=[0.5,1]$ and $[a_{\min},a_{\max}]=[-0.5,1]$.
\end{description}
Plots of alpha fair selection policies in these scenarios appear in each of Figs.~\ref{fig:Demo025}--\ref{fig:Demo08}.  The three figures respectively assume overall selection rates of $\sigma=0.25,0.6,0.8$.  These selection rates are chosen to be less than, equal to, and greater than a qualification rate of 0.6, which will be assumed for subsequent plots of alpha fair odds ratios and predictive rates.  

The plots show the relationship between alpha fair selection rates $(S,s)$ and the chosen value of $\alpha$.  As expected, larger values of $\alpha$ (indicating a greater emphasis on fairness) result in higher selection rates in the protected group (dashed curves) and lower rates in the nonprotected group (solid curves).  Scenario 1 calls for lower section rates in the protected group than Scenario~2 because of the greater utility cost of achieving fairness in Scenario~1; recall that alpha fairness consider total utility as well as Rawlsian fairness.  Both scenarios require selecting the entire protected group for sufficiently large $\alpha$, except when $\sigma=0.25$, in which case the small number of selections does not exhaust the protected group.  In Scenario~3, by contrast, the protected group's selection rate approaches 2/3 asymptotically, because only 2/3 of the group benefits from being selected in this scenario.

%These four subgroups are, of course, not homogeneous in reality, but we find that a homogeneity assumption nonetheless allows closed-form results that lead to interesting conclusions.  

%Accuracy is maximized by selecting all qualified individuals and rejecting all others.  

%Since protected individuals are less likely to be qualified, we cannot achieve fairness by rejecting those that are qualified.  We therefore assume that all qualified protected individuals are selected and all unqualified nonprotected individuals are rejected, a maneuver that drastically simplifies the algebra.  Another consequence that only the utility parameters $(a_0,b_0)$ and $(A_1,B_1)$ will play a role in our analysis, since the remaining parameters pertain to subgroups whose selection decisions are predetermined.  We can also suppose $A_1>0$, since selection of a qualified individuals surely increase utility.  However, it is quite possible that $a_0<1$, since selection of an unqualified individual may create a net reduction in utility, as in the case of a mortgage foreclosure, for example.  

\section{Demographic Parity}

Demographic parity is achieved when $P(D|\neg Z)=P(D|Z)$.  In the above model, this occurs when $s=S=\sigma$.  As it turns out, cases (a) and (b) of Theorem~\ref{th:alpha} do not apply, and we can achieve demographic parity only by choosing a value of $\alpha$ (if one exists) dictated by case (c).      

\begin{theorem}
An alpha fair selection policy for a given $\alpha$ results in demographic parity if and only if there exists a selection rate $S^*$ that satisfies the equation $\Delta_{S^*}(\alpha)=\Delta'_{S^*}(\alpha)$, in which case $(S^*,S^*)$ is such a policy.
\end{theorem}   

\begin{proof}
We first note as follows that neither case (a) nor (b) in Theorem~\ref{th:alpha} applies.  In case (a), demographic parity requires that 
\[
\min\Big\{1, \;\frac{\sigma}{1-\beta}\Big\} = \sigma, \;\; \mbox{or}\;\; \min\Big\{\frac{1}{\sigma}, \frac{1}{1-\beta}\Big\} = 1
\]
This cannot hold, because $\beta>0$ and $\sigma<1$.  Case (b) is similarly ruled out.  We are therefore left with case (c), wherein Theorem~\ref{th:alpha} implies that $S^*=s(S^*)$ if and only if $\Delta_{S^*}(\alpha)=\Delta'_{S^*}(\alpha)$, as claimed. $\Box$
\end{proof}

In Figs.~\ref{fig:Demo025}--\ref{fig:Demo08}, demographic parity is achieved at the value of $\alpha$ where the rising and falling curves for a given scenario intersect.  For example, if the overall section rate is $\sigma=0.6$, parity is achieved in Scenario~1 when $\alpha=0.833$ (Fig.~\ref{fig:Demo06}).  An important lesson in these plots is that a relatively small value of $\alpha$ frequently results in parity.  That is, parity achieves a rather modest degree of fairness when utilities are taken into account.  Indeed, proportional fairness ($\alpha=1$), which is something of an industrial benchmark, typically calls for selecting a significantly greater fraction of the protected group than the nonprotected group.  This is not the case in Scenario~3, however, where parity requires selecting protected individuals who receive minimal benefit and even harm from being selected.  For example, no value of $\alpha$ corresponds to parity when $\sigma=0.8$ (Fig.~\ref{fig:Demo08}) because alpha fairness never endorses harmful choices.

\begin{figure}[!h]
	\centering
	\includegraphics[scale=0.35,clip=true,trim=20 45 10 50]{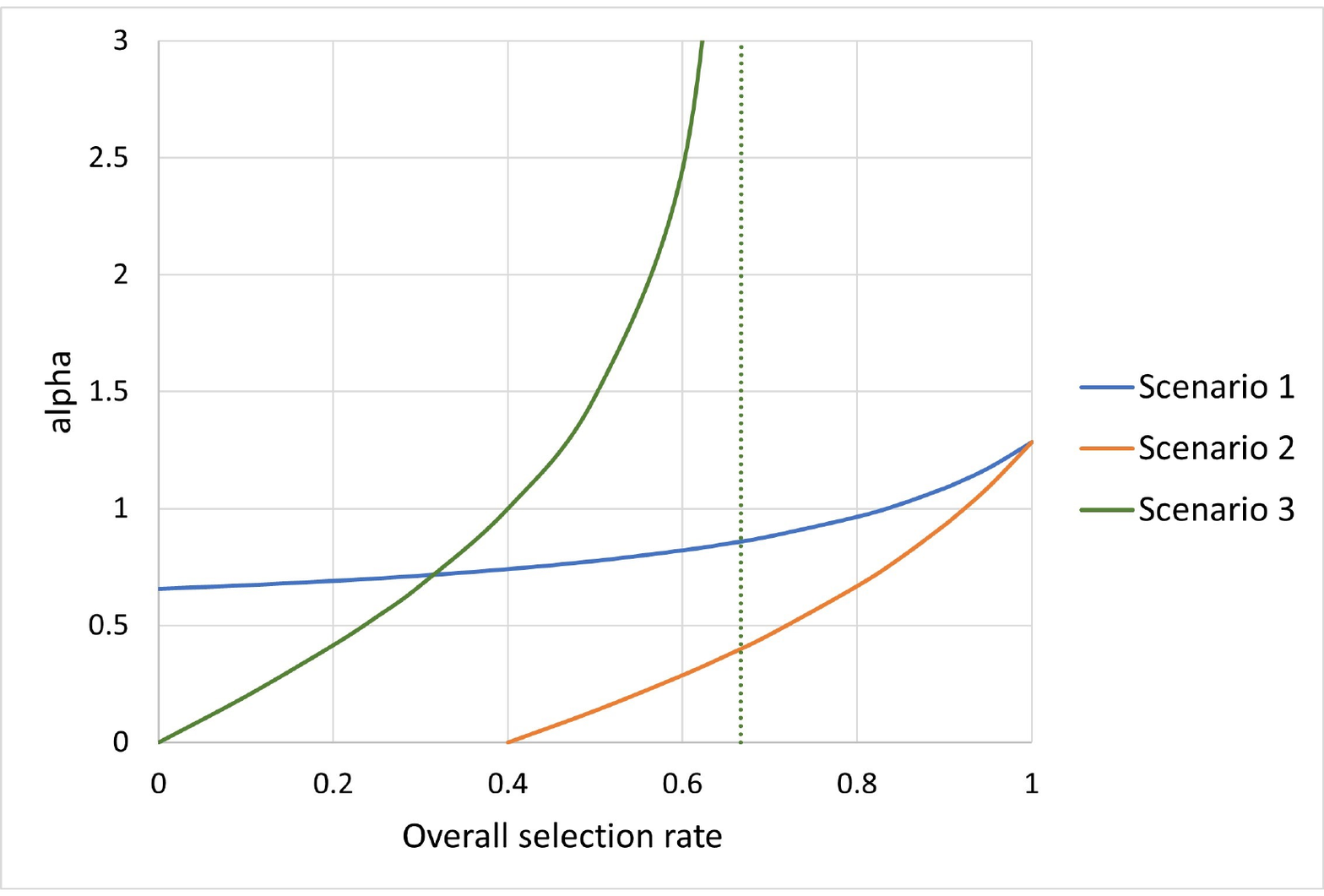} 
	\vspace{-3ex}
	\caption{Values of alpha that achieve demographic parity} \label{fig:Alpha}
\end{figure}

Figure~\ref{fig:Alpha} provides a fuller picture of the relation between the selection rate $\sigma$ and parity-achieving values of $\alpha$.  As $\sigma$ increases, parity corresponds to larger values of $\alpha$ because it becomes necessary to select protected individuals who benefit little from selection. The curves for Scenarios~1 and~2 happen to meet at $\sigma=1$ in this example because $A_{\min}$, $a_{\min}$, $B$, and $b$ are the same in the two scenarios.  We also note that $\alpha\rightarrow\infty$ as $\sigma\rightarrow 2/3$ in Scenario~3 because \mbox{$\sigma> 2/3$} requires selecting individuals who are harmed by selection.

Interestingly, smaller values of $\alpha$ correspond to parity in Scenario~2 than in Scenario~1, despite the fact that rejection can be quite costly to some members of the protected group in Scenario~2 (due to their higher selection benefits).  This occurs because a purely utilitarian assessment already takes this cost into account.  
%In fact, Fig.~\ref{fig:Alpha} shows that a purely utilitarian selection policy ($\alpha=0$) in Scenario~2 results in demographic parity when as few as 40\% of the population is selected, because a significant portion of this 40\% is drawn from the top of the protected group.    

\section{Equalized Odds}

Equalized odds are achieved when $P(D|Y,Z) = P(D|Y,\neg Z)$.  
To define equalized odds in the above model, we suppose that a fraction $Q$ of nonprotected individuals are qualified, and a fraction $q$ of protected individuals are qualified.  The a fraction $(1-\beta)Q+\beta q$ of the population is qualified.  We also make the reasonable assumption that the selection benefit is greater for qualified individuals than unqualified individuals within a given group.  Thus since $\Delta_S(\alpha)$ and $\Delta'_s(\alpha)$ are monotone decreasing as $S$ and $s$ increase, the qualified individuals in the nonprotected group consist of the fraction $Q$ with the largest welfare differentials.  The odds ratio for the nonprotected group is $S/Q$ when $S\leq Q$ and $1$ when $S>Q$, since in the latter case all the qualified individuals are selected.  Thus the odds ratio is $\min\{1,S/Q\}$ for the nonprotected group, and similarly for the protected group.  This means that we have equalized odds when 
\[
\min\Big\{ \frac{S}{Q},1\Big\} = \min\Big\{\frac{s}{q},1\Big\}
\]
This leads to the following theorem.  It is convenient to define $\rho$ to be the ratio of the fraction selected to the fraction of the population that is qualified, so that
\[
\rho = \frac{\sigma}{(1-\beta)Q + \beta q}
\]

\begin{theorem} \label{th:odds}
An alpha fair selection policy $(S^*,s(S^*))$ for a given $\alpha$ results in equalized odds if and only if one of the following holds:
\begin{align}
& S^* = Q\rho \leq Q \;\;\mbox{and}\;\; s(S^*) = q\rho \leq q \label{eq:odds1} \\
& S^*\geq Q \;\;\mbox{and} \;\; s(S^*)\geq q \label{eq:odds2} 
\end{align}
\end{theorem}

\begin{proof}
We consider four mutually exclusive and exhaustive cases:\\[1ex]
\begin{tabular}{l@{\hspace{5ex}}l}
	(a) $S^*\leq Q$ and $s(S^*)\leq q$ & (c) $S^*>Q$ and $s(S^*)\leq q$ \\
	(b) $S^*\leq Q$ and $s(S^*)>q$ & (d) $S^*>Q$ and $s(S^*)>q$
\end{tabular} \\[1ex]
In case (a), equalized odds is equivalent to $S^*(\alpha)/Q = s^*(\alpha)/q$, which implies $S^*=Q\rho$ and $s(S^*)=q\rho$ in \eqref{eq:odds1} due to \eqref{eq:alpha29}.  Conversely, we can see as follows that either of the conditions \eqref{eq:odds1} and \eqref{eq:odds2} implies equalized odds.  Under condition \eqref{eq:odds1}, the values for $S^*$ and $s(S^*)$ in \eqref{eq:odds1} imply $S^*/Q=s(S^*)/q$, and we have equalized odds.  Under condition \eqref{eq:odds2}, both odds ratios are 1, and we again have equalized odds.  In case (b), equalized odds implies $S^*=Q$, in which case condition \eqref{eq:odds2} is satisfied.  Conversely, the case hypothesis is consistent with only condition \eqref{eq:odds2}, in which case both odds ratios are 1 and we have equalized odds.  Case (c) is similar.  In case (d), one of the conditions \eqref{eq:odds1}--\eqref{eq:odds2} is necessarily satisfied (because the latter is satisfied), and we necessarily have equalized odds, because both odds ratios are 1.  $\Box$
\end{proof}

\begin{figure}[!h]
	\centering
	\includegraphics[scale=0.39,clip=true,trim=20 180 10 70]{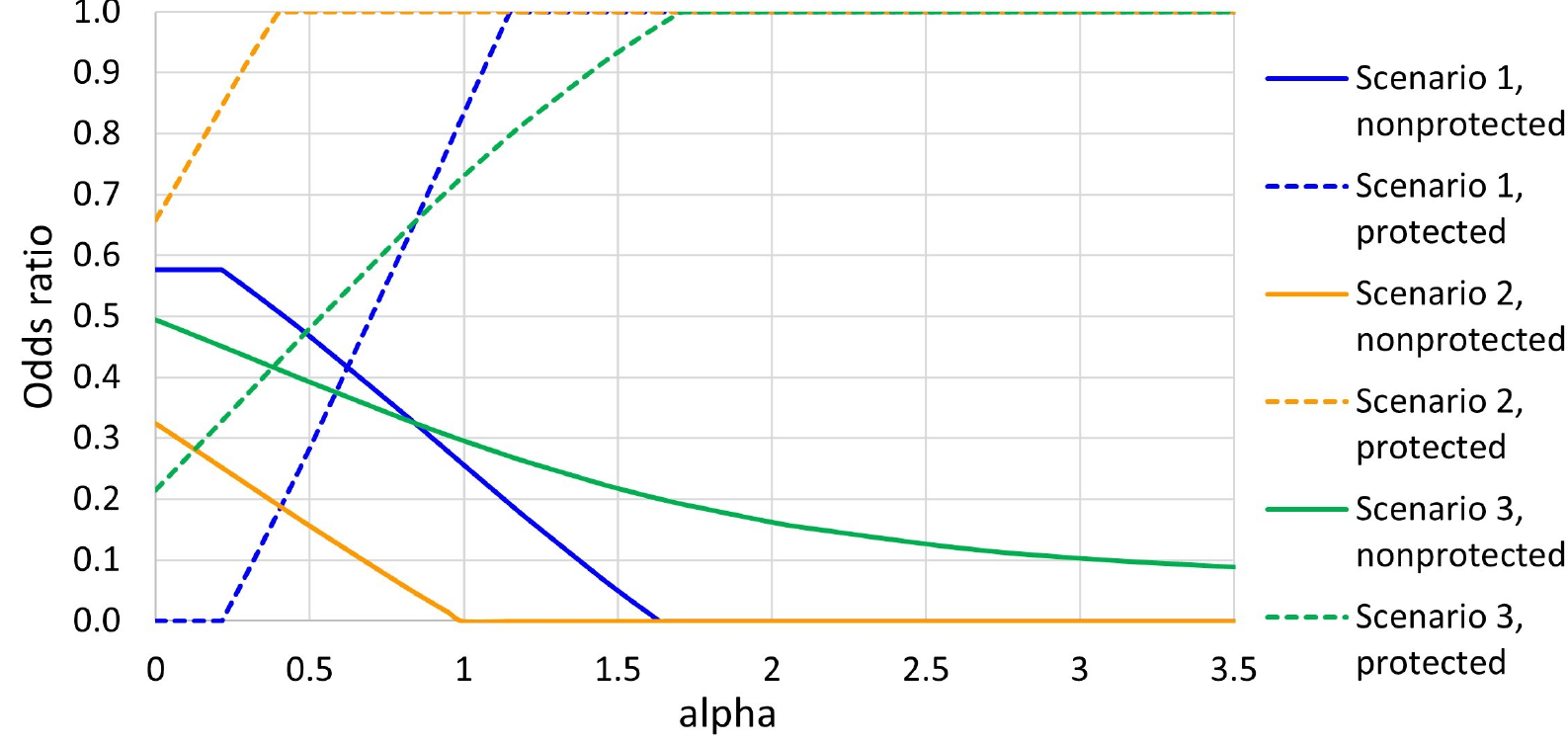} 
	\vspace{-3ex}
	\caption{Alpha fair odds ratios, assuming overall selection rate of 0.25.} \label{fig:Odds025}
\end{figure}

\begin{figure}[!h]
	\centering
	\includegraphics[scale=0.39,clip=true,trim=05 230 10 85]{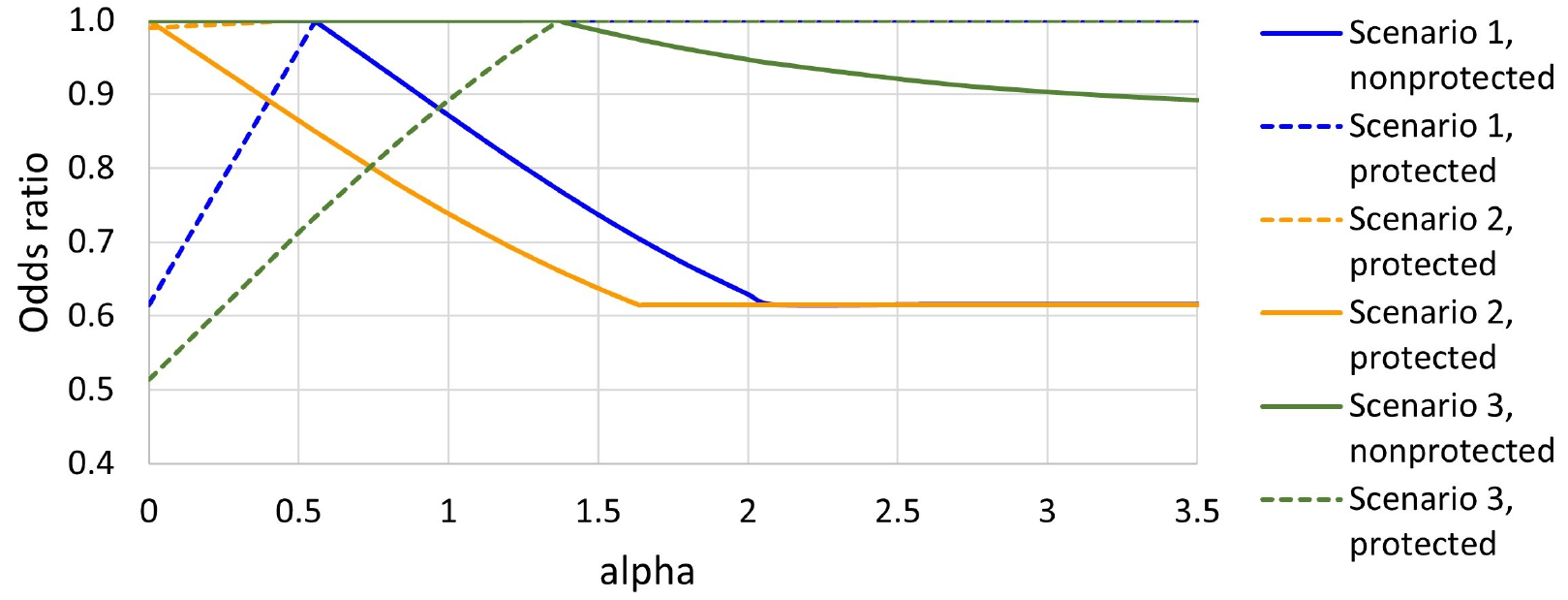}
	\vspace{-3ex}
	\caption{Alpha fair odds ratios, assuming overall selection rate of 0.6.} \label{fig:Odds06}
\end{figure}

To continue the example of the previous section, we suppose that the qualification rates are $(Q,q)=(0.65,0.5)$, so that a fraction $0.6$ of the population is qualified.  Figures~\ref{fig:Odds025} and~\ref{fig:Odds06}, corresponding to $\sigma=0.25$ and $\sigma=0.6$, show alpha fair odds ratios for various $\alpha$.  No plot is given for $\sigma=0.8$ because nearly all of the odds ratios are 1 due to the fact that considerably more individuals are selected than are qualified.  
%In fact, it would be unusual in practice to select a greater number of individuals than are qualified. 
In the important special case where the number selected is equal to the number qualified (Fig.~\ref{fig:Odds06}), equalized odds is achieved only by an accuracy-maximizing solution: precisely the qualified individuals are selected in both groups.  This rules out any adjustment for fairness.  The odds ratio is perhaps more useful when limited resources compel one to reject significantly many qualified individuals.  In this event, somewhat smaller values of $\alpha$ are typically necessary to achieve equalized odds than demographic parity (Fig.~\ref{fig:Odds025}).  In Scenario~2, a purely utilitarian solution already achieves a higher odds ratio for the protected group, since some of its qualified members derive more utility from selection than anyone in the nonprotected group.

\section{Predictive Rate Parity}

Predictive rate parity is achieved when $P(Y|D,Z) = P(Y|D,\neg Z)$.  
The predictive rate for the nonprotected group is $Q/S$ when $S\geq Q$ and $1$ when $S<Q$, since in the latter case all the selected individuals are qualified.  Thus the predictive rate is $\min\{Q/S,1\}$ for the nonprotected group, and similarly for the protected group.  This means that we have predictive rate parity when 
\[
\min\Big\{ \frac{Q}{S},1\Big\} = \min\Big\{\frac{q}{s},1\Big\}
\]
This leads to the following theorem, whose proof is very similar to the proof of Theorem~\ref{th:odds}.
\begin{theorem} \label{th:pred}
	An alpha fair selection policy $(S^*,s(S^*))$ for a given $\alpha$ results in predictive rate parity if and only if one of the following holds:
\begin{align}
	& S^* = Q\rho \geq Q \;\;\mbox{and}\;\; s(S^*) = q\rho \geq q \label{eq:pred1} \\
	& S^*\leq Q \;\;\mbox{and} \;\; s(S^*)\leq q \label{eq:pred2} 
\end{align}
Note that the expressions for $S^*$ and $s(S^*)$ in \eqref{eq:pred1} are the same as in \eqref{eq:odds1}.
\end{theorem}

Figures~\ref{fig:Pred06} and~\ref{fig:Pred08}, corresponding to $\sigma=0.6$ and $\sigma=0.8$, show alpha fair predictive rates for various $\alpha$.  There is no plot for $\sigma=0.25$, because nearly all of the predictive rates are 1.  
%The predictive rate is therefore meaningful only when one selects a greater number of individuals than are qualified ($\sigma>\rho$), which may be unlikely in practice.  
We also note that larger predictive rates correspond to {\em smaller} values of $\alpha$.  
%This means that if the protected group has a smaller predictive rate, this can be corrected only by making the selection process {\em less fair} as measured by alpha fairness.  This is perhaps to be expected, since a greater e,kphasisd on fairness tends to can require selecting less qualified memebrs of the proetcedt hgrup.  

\begin{figure}[!h]
	\centering
	\includegraphics[scale=0.39,clip=true,trim=05 230 10 85]{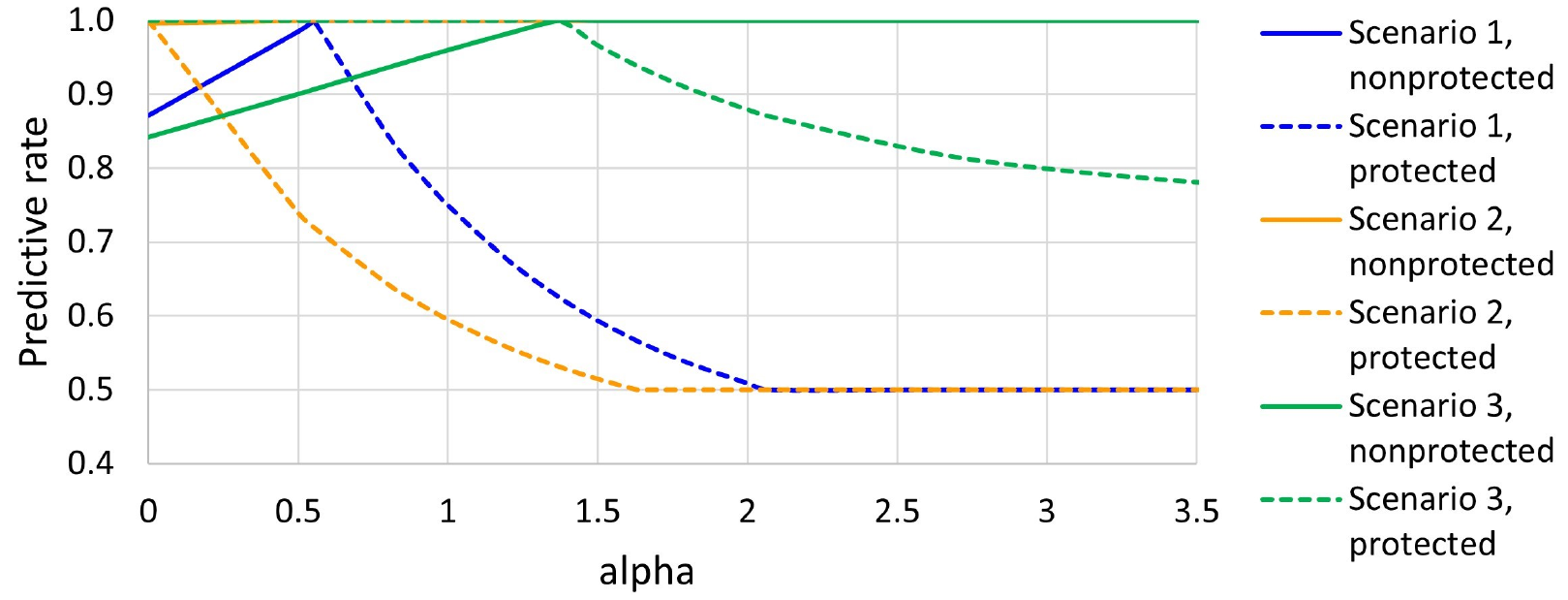}
	\vspace{-3ex}
	\caption{Alpha fair predictive rates, assuming overall selection rate of 0.6.} \label{fig:Pred06}
 \vspace{-6ex}
\end{figure}

\begin{figure}[!h]
	\centering
	\includegraphics[scale=0.39,clip=true,trim=05 215 10 100]{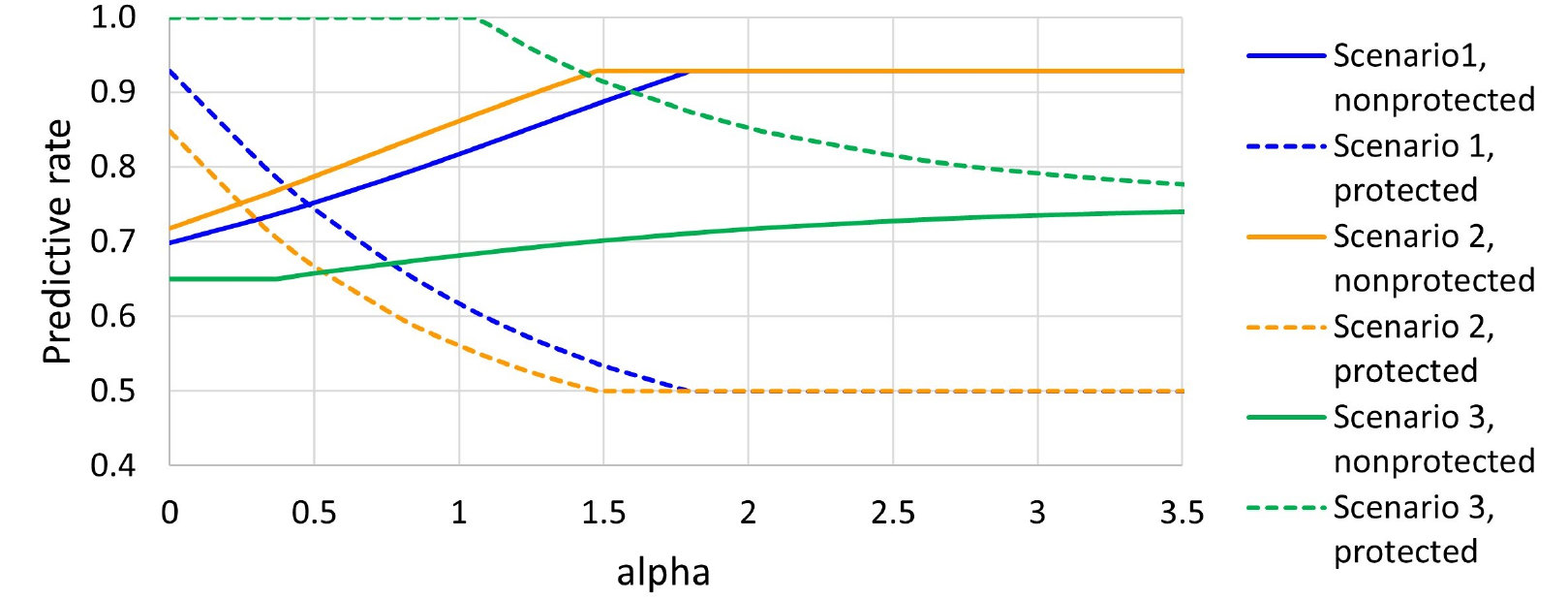}
	\vspace{-3ex}
	\caption{Alpha fair predictive rates, assuming overall selection rate of 0.8.} \label{fig:Pred08}
\end{figure}

\section{Conclusion}
	
Our aim in this paper has been to explore the extent to which social welfare optimization can assess well-known statistical parity metrics as criteria for group fairness in AI.  
%We examine demographic parity, equalized odds, and predictive rate parity as representative metrics.  
Our focus on alpha fairness allows us to address parity questions by appealing to a well-studied concept of just distribution with theoretical underpinning.  We conclude in this section by recalling  the two problems associated with parity metrics and summarizing how they might be addressed from an optimization perspective.
%Using our welfare-based approach to analyzing group parity metrics along with SWFs, we have arrived at some important new ways of resolving these problems:

\medskip
1. \emph{Accounting for welfare.} The alpha fairness criterion allows us to take explicit account of welfare implications, for various levels of fairness as indicated by the $\alpha$ parameter.  
We find that for certain values of $\alpha$ and certain group characteristics, an alpha fair selection policy can result in group parity of any of the three types.  Yet it can also call for significant statistical {\em disparity} in order to achieve an acceptable distribution of utilities.  

In particular, the alpha values that result in parity typically lie significantly below that corresponding to proportional fairness ($\alpha=1$)---except when some individuals in the protected group are actually harmed by being selected, in which case larger values of $\alpha$ correspond to parity.  Since proportional fairness is the most widely defended and applied variety of alpha fairness, it is noteworthy that it often requires, not parity, but {\em higher} selection rates for the protected group than for the rest of the population.  
%This reflects the fact that alpha fairness takes into account the {\em ex ante} disadvantaged state of the protected group.  
In addition, a lower level of fairness (i.e., a smaller $\alpha$) is necessary to achieve parity when rejection is more costly to members of the protected group than the rest of the population, other things being equal.  This is because even a purely utilitarian accounting already takes this cost into account.  

\medskip
2. \emph{Selecting and justifying parity metrics.}  We derive a number of conclusions regarding the choice of parity metric.  In general, we find that the implications of alpha fairness depend heavily on how many individuals are selected relative to the total number qualified, at least where equalized odds and predictive rate parity are concerned.  
%This relationship is not directly relevant to demographic parity, however, as it is not defined in terms of qualifications.

To elaborate on this, we first suppose that the total number selected is the same (or approximately the same) as the total number who are qualified in the population as a whole.  In this case, demographic parity follows the pattern described above, in which relatively small values of $\alpha$ result in parity, except when some protected individuals are harmed by selection.  Yet equalized odds, as well as predictive rate parity, are achieved if and only if the odds ratios and the predictive rates are 1 in both groups.  This corresponds to an accuracy maximizing policy of selecting all and only qualified individuals.
%(a policy that achieves alpha fairness at the same value of $\alpha$ for both types of parity).  
As a result, neither equalized odds nor predictive rate parity reflects any consideration of fairness beyond mere accuracy, and consequently neither is suitable as a fairness criterion in this context.  
%Demographic parity, by contrast, is a meaningful criterion but often corresponds to a relatively low degree of fairness.

We next suppose that the total number selected is significantly less than the number qualified, presumably a common situation due to limited resources.  In this case, equalized odds is generally achieved for smaller values $\alpha$ than are required for demographic parity, considerably smaller when some protected individuals are harmed by selection.  This indicates that demographic parity demands a greater emphasis on fairness than equalized odds.
%albeit still a modest amount in most situations.  
This is consonant with the fact that equalized odds is sometimes seen as more easily defended, perhaps on grounds of equality of opportunity, than is demographic parity, which may reflect a desire to compensate for historically unjust discrimination.  As for the predictive rate, it is almost always 1 when a significant number of qualified individuals are rejected, since those who make it through the sieve are almost always qualified.  This means that predictive rate parity is likely to be achieved simply due to the high rejection rate and is therefore of little value as a fairness criterion.

Finally, we suppose that the number selected is significantly greater than the number qualified.  Here, the odds ratio loses interest because it is almost always~1.  While predictive rate parity becomes meaningful in this case, decision makers may be reluctant to select more individuals than are qualified.  To the extent this is true, predictive rate parity has limited usefulness.  A possible exception arises in the controversy over parole mentioned earlier.  Predictive rate parity might be defended on the ground that a lower recidivism rate in the protected group (and therefore a higher predictive rate) may reflect stricter parole criteria than for other inmates \cite{AnwFan15}.  Greater fairness may therefore require a {\em reduction} in the predictive rate of the protected group, which we have seen can be achieved by choosing a larger value of $\alpha$.  If this is taken as justifying a practice of paroling more individuals than are qualified (perhaps in order to reduce the predictive rate of protected individuals without tightening the criteria for others), then predictive rate parity could be a suitable criterion.  

In summary, demographic parity can under certain conditions correspond to an alpha fair policy, but it may result in less fairness than desired for the protected group.  Equalized odds can be a useful criterion when fewer individuals are selected than are qualified to be selected, but it corresponds to an even lesser degree of fairness.  Predictive parity is a meaningful fairness measure only in the perhaps rather uncommon situation when decision makers select significantly more individuals than are qualified.
\medskip

The foregoing conclusions regarding equalized odds and predictive rate parity rest on the assumption that, {\em within a given group}, qualified individuals are selected before unqualified individuals.  This assumption might be defended on the ground that (a) qualified individuals are likely to benefit more from being selected, and (b) individuals who benefit more from being selected are selected first in the group.  Assumption (a) might be based on observations that less qualified individuals pose a greater risk of defaulting on a mortgage, failing to secure a job, committing a crime while on parole, and so forth, and therefore have less expected benefit.  As for (b), there is no apparent rationale, based on either expected utility or fairness, for selecting individuals within a group in any other order.  It therefore seems reasonable to suppose (b) is true before assessing fairness.  
%It is on this basis that we assume, for the purpose of assessing fairness, that qualified individuals within a given group are selected first in that group.   

We believe these results suggest that there is potential in an  optimization perspective to inform fairness debates in AI.  Further research could explore the parity implications of alternative social welfare functions, such as the Kalai-Smorodinsky and threshold criteria cited earlier.  A particularly interesting research issue is the extent to which achieving fairness in the population as a whole can result in a reasonable degree of parity across all groups.  This would obviate the necessity of selecting which groups to regard as protected, and how to balance their interests.

\bibliographystyle{splncs04}
%\bibliography{Bibliography,paperref}

%%%%%%%%%%%%%%%%%%%%%%%%%%%%%%%%%%%%%%%%%%%%%%%%%%%%%%%%%%%%

\end{document}